\documentclass[12pt]{article}

\usepackage{pst-all}
\usepackage{multido}

\usepackage{amsmath}
\usepackage{times}
\usepackage{graphicx}
\usepackage{color}
\usepackage{multirow}
\usepackage[authoryear]{natbib}
\usepackage{rotating}
\usepackage{bbm}
\usepackage{latexsym}

\usepackage{dsfont}
\usepackage{amsthm}

\newtheorem{theorem}{Theorem}
\newtheorem{lemma}[theorem]{Lemma}
\newtheorem{corollary}[theorem]{Corollary}

\newtheorem*{thma}{{\em Theorem~2 in~\cite{Roux:08}}}
\newtheorem*{thmb}{{\em Theorem~4 in~\cite{Roux:10}}}
\newtheorem*{thmc}{{\em Theorem~2 in~\cite{Roux:10}}}
\newcommand{\ha}{{\hat j}}

\newcommand{\captionfonts}{\normalsize}

\makeatletter  
\long\def\@makecaption#1#2{%
  \vskip\abovecaptionskip
  \sbox\@tempboxa{{\captionfonts #1: #2}}%
  \ifdim \wd\@tempboxa >\hsize
    {\captionfonts #1: #2\par}
  \else
    \hbox to\hsize{\hfil\box\@tempboxa\hfil}%
  \fi
  \vskip\belowcaptionskip}
\makeatother   
\renewcommand{\thefootnote}{\normalsize \arabic{footnote}} 	

\begin{document}
\hspace{13.9cm}1

\ \vspace{20mm}\\

\noindent {\LARGE Refinements of Universal Approximation Results for Deep Belief Networks and Restricted Boltzmann Machines}

\ \\
{\bf \large Guido Montufar,$^{\displaystyle 1,\ast}$} {\bf \large Nihat Ay$^{\displaystyle 1, \displaystyle 2}$}\\
\\
{$^{\displaystyle 1}$Max Planck Institute for Mathematics in the Sciences,  Inselstra\ss e 22, D-04103 Leipzig, Germany.}\\
{$^{\displaystyle 2}$Santa Fe Institute, 1399 Hyde Park Road, Santa Fe, New Mexico 87501, USA.}%\\
%\ \\[-2mm]
%{\bf Keywords:} Manuscript, journal, instructions

\renewcommand{\thefootnote}{\fnsymbol{footnote}} 
\footnotetext{* montufar@mis.mpg.de}

\thispagestyle{empty}
%\markboth{}{NC instructions}
%
%\ \vspace{-0mm}\\
%
%Abstract
\begin{abstract}
We improve recently published results about resources of Restricted Boltzmann Machines (RBM) and Deep Belief Networks (DBN) required to make them Universal Approximators. 
We show that any distribution $p$ on the set $\{0,1\}^n$ of binary vectors of length $n$ can be arbitrarily well approximated by an RBM with $k-1$ hidden units, where $k$ is the minimal number of pairs of binary vectors differing in only one entry such that their union contains the support set of $p$. In important cases this number is half of the cardinality of the support set of $p$ (given in \cite{Roux:08}). 
We construct a DBN with $\frac{2^n}{2(n-b)}$, $b\sim \log n$, hidden layers of width $n$ that is capable of approximating any distribution on $\{0,1\}^n$ arbitrarily well. This 
confirms a conjecture presented in \cite{Roux:10}.  
\end{abstract}

\section{Introduction}

This work rests upon ideas presented in \cite{Roux:08} and \cite{Roux:10}. We positively resolve a conjecture that was posed in \cite{Roux:10}. 
Before going into the details of this conjecture we first recall some basic notions.

The definition of RBM's and DBN's that we use is the one given in the papers mentioned above and references therein. For details the reader is referred to those works. Here we give a short description: A Boltzmann Machine consists of a collection of binary stochastic units, where any pair of units may interact. The unit set is divided into {\em visible} and {\em hidden} units. Correspondingly the state is characterized by a pair $(v,h)$ where $v$ denotes the state of the visible and $h$ denotes the state of the hidden units.  One is usually 
interested in 
distributions on the visible states $v$ and would like to generate these as marginals of distributions on the states $(v,h)$. In a general Boltzmann Machine 
the interaction graph is allowed to be complete. 
A Restricted Boltzmann Machine (RBM) is a special type of  Boltzmann Machine, where the graph describing the interactions is bipartite: Only connections between 
visible and hidden units appear. It is not allowed that two visible units or two hidden units interact with each other (see  Fig.~\ref{graphRBMandDBN}).  
The distribution over the states of all RBM units has the form of the Boltzmann distribution $p(v,h)\propto \exp(h^T W \cdot v + B\cdot v + C\cdot h)$, where $v$ is a binary vector of length equal to the number of visible units, and $h$ a binary vector with length equal to the number of hidden units. The parameters of the RBM are given by the matrix $W$ and the two vectors $B$ and $C$. 
A Deep Belief Network consists of a chain of layers of units. Only units from neighboring layers are allowed to be connected, there are no connections within each layer. The last two layers have undirected connections between them, while the other layers have connections directed towards the first layer, the visible layer. 
The general idea of a DBN is to assume that all layers are of similar size, as shown in Fig.~\ref{graphRBMandDBN}. 

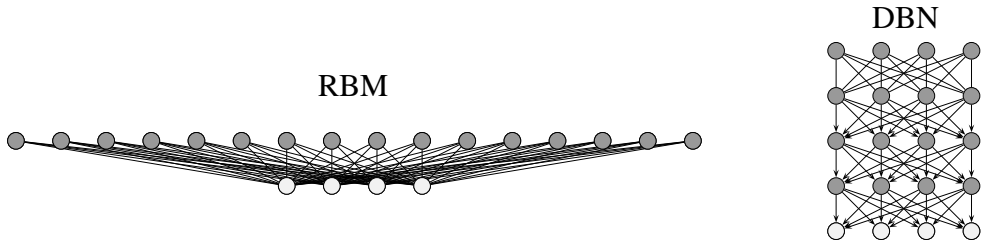
\begin{figure}

\vfill
\psset{unit=.6cm}
\psset{linewidth=.3pt}

% RBM
\begin{pspicture}(-1,-2)(15,2)
{\psset{fillstyle=solid,fillcolor=black!40}
  \multido{\n=0+1}{16}{%
  \cnodeput(\n,0){\n}{}
   }
  \multido{\n=6+1}{4}{%
  \cnodeput(\n,-1){-\n}{}
   }
}

\multido{\i=6+1}{4}{
  \multido{\n=0+1}{16}{%
  \ncline{\n}{-\i} }
 }
{\psset{fillstyle=solid,fillcolor=black!5}
  \multido{\n=6+1}{4}{%
  \cnodeput(\n,-1){-\n}{}
   }
}
{\psset{fillstyle=solid,fillcolor=black!40}
  \multido{\n=0+1}{16}{%
  \cnodeput(\n,0){\n}{}
   }
}
\put(6.7,1){RBM}
\end{pspicture}
%
% DBN 
\begin{pspicture}(-2,0)(5,5)
\put(1.8,4.5){DBN}
{\psset{fillstyle=solid,fillcolor=black!40}
 \cnodeput(1,0){01}{}
 \cnodeput(2,0){02}{}
 \cnodeput(3,0){03}{}
 \cnodeput(4,0){04}{}   
 \cnodeput(1,1){11}{}
 \cnodeput(2,1){12}{}
 \cnodeput(3,1){13}{}
 \cnodeput(4,1){14}{}   
 \cnodeput(1,2){21}{}
 \cnodeput(2,2){22}{}
 \cnodeput(3,2){23}{}
 \cnodeput(4,2){24}{}   
 \cnodeput(1,3){31}{}
 \cnodeput(2,3){32}{}
 \cnodeput(3,3){33}{}
 \cnodeput(4,3){34}{}   
 \cnodeput(1,4){41}{}
 \cnodeput(2,4){42}{}
 \cnodeput(3,4){43}{}
 \cnodeput(4,4){44}{}   
}

\ncline{<-}{01}{11}
\ncline{<-}{01}{12}
\ncline{<-}{01}{13}
\ncline{<-}{01}{14}
\ncline{<-}{02}{11}
\ncline{<-}{02}{12}
\ncline{<-}{02}{13}
\ncline{<-}{02}{14}
\ncline{<-}{03}{11}
\ncline{<-}{03}{12}
\ncline{<-}{03}{13}
\ncline{<-}{03}{14}
\ncline{<-}{04}{11}
\ncline{<-}{04}{12}
\ncline{<-}{04}{13}
\ncline{<-}{04}{14}

\ncline{<-}{11}{21}
\ncline{<-}{11}{22}
\ncline{<-}{11}{23}
\ncline{<-}{11}{24}
\ncline{<-}{12}{21}
\ncline{<-}{12}{22}
\ncline{<-}{12}{23}
\ncline{<-}{12}{24}
\ncline{<-}{13}{21}
\ncline{<-}{13}{22}
\ncline{<-}{13}{23}
\ncline{<-}{13}{24}
\ncline{<-}{14}{21}
\ncline{<-}{14}{22}
\ncline{<-}{14}{23}
\ncline{<-}{14}{24}

\ncline{->}{31}{21}
\ncline{->}{31}{22}
\ncline{->}{31}{23}
\ncline{->}{31}{24}
\ncline{->}{32}{21}
\ncline{->}{32}{22}
\ncline{->}{32}{23}
\ncline{->}{32}{24}
\ncline{->}{33}{21}
\ncline{->}{33}{22}
\ncline{->}{33}{23}
\ncline{->}{33}{24}
\ncline{->}{34}{21}
\ncline{->}{34}{22}
\ncline{->}{34}{23}
\ncline{->}{34}{24}

\ncline{31}{41}
\ncline{31}{42}
\ncline{31}{43}
\ncline{31}{44}
\ncline{32}{41}
\ncline{32}{42}
\ncline{32}{43}
\ncline{32}{44}
\ncline{33}{41}
\ncline{33}{42}
\ncline{33}{43}
\ncline{33}{44}
\ncline{34}{41}
\ncline{34}{42}
\ncline{34}{43}
\ncline{34}{44}

{\psset{fillstyle=solid,fillcolor=black!5}
  \cnodeput(1,0){01}{}
 \cnodeput(2,0){02}{}
 \cnodeput(3,0){03}{}
 \cnodeput(4,0){04}{}   
}
\end{pspicture}

\caption{In the left side we sketched the graph of interactions in an RBM, in the right side the corresponding graph for a DBN with $n=4$ visible units (drawn brighter). An arbitrary weight can be assigned to every edge. Beside this connection weights, every node contains an individual {\em offset} weight. Every node takes value $0$ or $1$ with a probability that depends on the weights. The RBM and DBN of size depicted above are examples of universal approximators of distributions on $\{0,1\}^4$ (\cite{Roux:08} and \cite{Roux:10} respectively). In the present paper is shown that the number of hidden units in the RBM can be halved, and the number of hidden layers in the DBN can be roughly halved. }\label{graphRBMandDBN}
\end{figure}

A major difficulty in the use of Boltzmann Machines always has been the slowness of learning. In order to overcome this problem, DBN's have been proposed as an alternative to classical Boltzmann Machines.  
An efficient learning algorithm for DBN's was given in the paper~\cite{Hinton:2006}.

The fundamental questions along the above-mentioned previous work are the following: Does a DBN exist that is capable of approximating any distribution on the visible states through appropriate choice of parameters? We will refer to such a DBN as a universal DBN approximator (similarly we will use the denomination universal RBM approximator). If universal DBN approximators exist, what is their minimal size? 

Since DBN's are more difficult to study than RBM's, as a preliminary step, corresponding questions related to the representational power of RBM's have been addressed. 
Theorem~2 in~\cite{Roux:08} shows that any distribution on $\{0,1\}^n$ with support of cardinality $s$ is arbitrarily well approximated (with respect to the Kullback Leibler divergence) by the marginal distribution of an RBM containing $s +1$ hidden units:

\begin{thma}%[{\em \bf Theorem~2 in~\cite{Roux:08}}]
Any distribution on $\{0,1\}^n$ can be approximated arbitrarily well with an RBM with $s+1$ hidden units, where $s$ is the number of input vectors whose probability does not vanish.
\end{thma}

This theorem proved the existence of a universal RBM approximator. The existence proof of a universal DBN approximator is due to  
\cite{Hinton:2008}. More precisely, \cite{Hinton:2008} explicitely constructed a DBN with $\sim 3\cdot 2^n$ hidden layers of width $n+1$ that approximates any distribution 
on $\{0,1\}^n$. 
Given that the existence problem of universal DBN approximators was positively resolved through this result, the efforts have been put into optimizing the size, i.e. reducing the number of parameters. This can be done by reducing the number of hidden layers involved in a DBN, or by making the hidden layers narrower. 
In terms of simple counting arguments, we give a lower bound on the minimal number of hidden layers required for the universality of a DBN with layers of size $n$. 
The number of free parameters in such a DBN is {\em square of the width of each layer} $\times$ {\em number of hidden layers} $+$ {\em number of units}, which for $k$ hidden layers is $k (n^2 + n)+ n$. On the other hand, the number of parameters needed to describe all distributions on $2^n$ elements, e.g. over binary vectors of length $n$, is $2^n-1$. Therefore, a lower bound on the number of hidden layers of a 
universal DBN approximator is given by $\tfrac{2^n-1-n}{n(n+1)}$ (which yields $2^n-1$ free parameters). Otherwise the number of parameters would not be sufficient.
Asymptotically, this bound is of order $\frac{2^n}{n^2}$.
Certainly, since the architecture of DBN's makes important restrictions on the way the parameters are used, such a lower bound is not necessarily %expected to be 
achievable. In particular the approximation of a distribution through a DBN or RBM is not unambiguous, i.e. for several choices of the parameters the same distribution is produced as marginal distribution. 
However, 
 in~\cite{Roux:10} it has been shown that a number of hidden layers of order $\frac{2^n}{n}$ is sufficient:  

\begin{thmb}%[Theorem~4 in \cite{Roux:10}]
If $n=2^t$, a DBN composed of $\frac{2^n}{n}+1$ layers of size $n$ is a universal approximator of distributions on $\{0,1\}^n$.
\end{thmb}
 
In the paper \cite{Roux:10} the optimality of the bound given in this theorem remains an open problem. However, their proof method suggests the sufficiency of less 
hidden layers, which was conjectured in their paper. The proof of Theorem 4 crucially depends on the authors' previous Theorem~2 in~\cite{Roux:08}. Our main contribution is to sharpen Theorem~2 (see Theorem~\ref{universalRBM} in Section~\ref{s:results}) which allows us to even better exploit their method 
and thereby confirm their conjecture (see Theorem~\ref{universalDBN} in Section~\ref{s:results}). %We consider our refinement as particularly interesting because there are reasons to believe that this already provides the optimal bound for the minimal number of hidden layers in a universal DBN approximator. 

\section{Results}\label{s:results}

%------------------------------------------%
\subsection{Restricted Boltzmann Machines}
%------------------------------------------%
The following Theorem~\ref{universalRBM} sharpens Theorem~2 in~\cite{Roux:10}. We will use it (its Corollary~\ref{corollaryRBM}) in the proof of our main result, Theorem~\ref{universalDBN}. 

\begin{theorem}[Reduced RBM's which are universal approximators]\label{universalRBM}
Any distribution $p$ on binary vectors of length $n$ can be approximated arbitrarily well by an RBM with $k-1$ hidden units, where $k$ is the minimal number of pairs of binary vectors, such that the two vectors in each pair differ in only one entry, and such that the support set of $p$ is contained in the union of these pairs.
\end{theorem}

%------------------------------------------%
The set $\{0,1\}^n$ corresponds to the vertex set of the $n$-dimensional cube. The edges of the $n$-dimensional cube correspond to pairs of binary vectors of length $n$ which differ in exactly one entry. For the graph of the $n$-dimensional cube there exist perfect matchings, i.e., collections of disjoint edges which cover all vertices. Therefore we have the following: 
\begin{corollary}%[of Th.~\ref{universalRBM}]
\label{corollaryRBM}
Any distribution on $\{0,1\}^n$ can be approximated arbitrarily well by an RBM with $\frac{2^n}{2}-1$ hidden units. 
\end{corollary}

The proof of Theorem~\ref{universalRBM} given below is very much in the spirit of the proof of Theorem~2 in~\cite{Roux:08}. 
The idea there consists on showing that given an RBM with some marginal visible distribution, the inclusion of an additional hidden unit allows to increment the probability mass of one visible state vector, while uniformly reducing the probability mass of all other visible vectors. 

We show that the inclusion of an additional hidden unit in fact allows to increase the probabiliy mass of a pair of visible vectors, in independent ratio, given that this pair differs in one entry. At the same time, the probability of all other visible states is reduced uniformly. We also use the offset weights in the visible units to further improve the result.

\begin{proof}[Proof of Theorem~\ref{universalRBM}] We stay close to the notation used in \cite{Roux:08}. 
1. Let $p$ be the distribution on the states of visible and hidden units of an RBM. Its marginal probability distribution on $v$ can be written as
\begin{equation*}
p(v)=\frac{\sum_{h} z(v,h)}{\sum_{v^0,h^0}z(v^0,h^0)}.
\end{equation*}
Denote by $p_{w,c}$ the distribution arising through the adding of a hidden unit to the RBM connected with weigths $w=(w_1,\ldots,w_n)$ to the visible units, and with offset weight $c$. Its marginal distribution can be written as
\begin{equation*}
p_{w,c}(v)=\frac{(1+\exp (w\cdot v +c))\sum_{h} z(v,h)}{\sum_{v^0,h^0}(1+\exp (w\cdot v^0 +c))z(v^0,h^0)}.
\end{equation*}

2. Given any vector $v\in\{0,1\}^n$ we write $v_\ha$ for the vector defined through $(v_\ha)_i=v_i, \forall i\neq j$, and $(v_\ha)_j=0$. We also write ${\mathds{1}}:=(1,\ldots,1)$, and $e_j:={\mathds{1}}-{\mathds{1}}_\ha$.

3. For any $j\in\{1,\ldots,n\}$ let $\tilde v$ be an arbitrary vector with $\tilde v_j=1$, and $s:=|\{i\neq j:\tilde v_i=1\}|$. Define 
\begin{eqnarray*}
\hat w &:=& a(\tilde v_\ha -\frac{1}{2}{\mathds{1}}_\ha),\\
\bar w &:=& a(\tilde v_\ha -\frac{1}{2}{\mathds{1}}_\ha) +(\lambda_2 -\lambda_1) e_j,\\
\bar c &:=& -\hat w\cdot \tilde v + \lambda_1 = -\hat w \cdot \tilde v_\ha + \lambda_1.
\end{eqnarray*}
For the weights $\bar w$ and $\bar c$ we have:
\begin{eqnarray*}
\bar w \cdot v &=& \frac{1}{2} a(s - |\{i:(\tilde v_\ha)_i\neq (v_\ha)_i\}|) +(\lambda_2-\lambda_1) v_j,\\
\bar c &=&-\frac{1}{2} a s +\lambda_1,
\end{eqnarray*}
and in the limit $a\to\infty$ we get:
\begin{eqnarray*}
\lim_{a\to\infty} 1 +\exp(\bar w\cdot v + \bar c) &=& 1,\quad\forall v\neq \tilde v, \tilde v_\ha, \\
\lim_{a\to\infty} 1 +\exp(\bar w\cdot \tilde v_\ha + \bar c) &=& 1 + e^{\lambda_1},\\
\lim_{a\to\infty} 1 +\exp(\bar w\cdot \tilde v + \bar c) &=& 1 + e^{\lambda_2}.
\end{eqnarray*}
Just as in the Proof of Theorem 2 in \cite{Roux:08} this yields for the marginal distribution on the visible states of the enlarged RBM the following:
\begin{eqnarray*}
\lim_{a\to\infty} p_{\bar w,\bar c} (v) &=& \frac{p(v)}{1+e^{\lambda_1} p(\tilde v_\ha) + e^{\lambda_2} p(\tilde v)},\quad\forall v\neq \tilde v,\tilde v_\ha,\\
\lim_{a\to\infty} p_{\bar w,\bar c} (\tilde v_\ha) &=& \frac{(1+e^{\lambda_1})p(\tilde v_\ha)}{1+e^{\lambda_1} p(\tilde v_\ha) + e^{\lambda_2} p(\tilde v)},\\
\lim_{a\to\infty} p_{\bar w,\bar c} (\tilde v) &=& \frac{(1+e^{\lambda_2})p(\tilde v)}{1+e^{\lambda_1} p(\tilde v_\ha) + e^{\lambda_2} p(\tilde v)}.
\end{eqnarray*}
This means that the probability of $\tilde v$ and of $\tilde v_{\hat j}$ can be increased independently by a multiplicative factor, while all other probabilities are reduced uniformly. 

4. Now we explain how to start an induction from which the claim follows. 
Consider an RBM with no hidden units, RBM$^0$. Through a choice of the offset weigths in every visible unit, RBM$^0$ produces as visible distribution any arbitrary factorizable distribution $p^0(v)\propto \exp(B\cdot v) \propto \exp(B\cdot v + K)$, where $B$ is the vector of offset weights and $K$ is a constant that we introduce for illustrative reasons, and is not a parameter of the RBM$^0$ since it cancels out with the normalization of $p^0$.  In particular, RBM$^0$ can approximate arbitrarily well any distribution with support given by a pair of vectors that differ in only one entry. To see this consider any pair of vectors $\tilde v$ and $\tilde v_{\hat j}$  that differ in the entry $j$. Then, the choice $B=a(\tilde v_\ha -\frac{1}{2}{\mathds{1}}_\ha) +(\lambda_2 -\lambda_1) e_j$ and $K=-a (\tilde v_\ha -\frac{1}{2}{\mathds{1}}_\ha) \tilde v + \lambda_1$ yields in the limit $\lim_{a\to\infty}$ (similarly to the equations in item 3. above) that $\lim_{a\to\infty}p^0(v)=0$ whenever $v\neq\tilde v$ and $v\neq \tilde v_\ha$, while $\lim_{a\to\infty} p^0(\tilde v) /p^0(\tilde v_\ha)=\exp(\lambda_2-\lambda_1)$ can be chosen arbitrarily by modifying $\lambda_1$ and $\lambda_2$. Hence, $p^0$ can be made arbitrarily similar to any distribution with support $\{\tilde v,\tilde v_\ha\}$. Notice that $p^0$ remains positive for all $v$ and $a<\infty$. 

By the arguments described above, every additional hidden unit allows to increase the probability of any pair of vectors which differ in one entry. Obviously, it is possible to do the same for a single vector instead of a pair. Thence, with every additional hidden unit the support set of the probabilities which can be approximated arbitrarily well is enlarged by an arbitrary pair of vectors which differ in one entry. This is, RBM$^{(i-1)}$ is an approximator of distributions with support contained in any union of $i$ pairs of vectors which differ in exactly one entry. 
\end{proof}

We close this passage with some remarks: 

The possiblity of independent change of the probability mass of two visible vectors is due to the usability of the following two parameters: a) The offset input weigth in the added hidden unit, and b) the weight of the connection between the added hidden unit and the visible unit where the pair of visible vectors differ. See item 3. in the Proof. 

The attempt to use a similar idea to increment the probability mass of three different vectors in independent ratios inducts a coupled change in the probability of a fourth vector. Three vectors differ in at least 2 entries, as do four vectors. Since only 3 parameters are available (the offset of the new hidden unit and two connection weigths), the dependency arises. 

It is worth noting, that using exclusively a similar idea will not allow an exension of Theorem~2 in~\cite{Roux:10} to permit the flip of a certain bit with a certain probability (only) given one of three input vectors.

%------------------------------------------%
\subsection{Deep Belief Networks}\label{s:DBN}
%------------------------------------------%

In this section we implement our Theorem~\ref{universalRBM} to modify the construction given in the proof of Theorem~4 in \cite{Roux:10} and prove our main result, Theorem~\ref{universalDBN}:

\begin{theorem}[Reduced DBN's which are universal approximators]\label{universalDBN}
Let $n=\frac{2^b}{2}+b$, $b\in{\bf N}$, $b\geq 1$. A DBN containing $\frac{2^n}{2(n-b)}$ hidden layers of width $n$ is a universal approximator of distributions on $\{0,1\}^n$.
\end{theorem}

Before proving Theorem~\ref{universalDBN} we first develop some components of the proof. 

An important idea of \cite{Hinton:2008} is that of {\em sharing}, by means of which in a part of a DBN the probability of a vector is increased while the probability of another vector is decreased and the probability of all other vectors remains nearly constant. This idea is refined in Theorem~2 of \cite{Roux:10}: 

\begin{thmc}%[Theorem~2 in~\cite{Roux:10}]
[slightly different formulation]
Consider two layers of units indexed by $i\in\{1,\ldots,n\}$ and $k\in\{1,\ldots,n\}$, and denote by $v$ and $h$ state vectors in each layer. Denote by $\{w_{ik}\}_{i,k=1,\ldots,n}$ the connection weights and by $\{c_k\}_{k=1,\ldots,n}$ the offset weights in the second layer. 
Given any $l$ and $j$, $l\neq j$, let $a$ be an arbitrary vector in $\{0,1\}^n$ and $b$ another vector with $b_i=a_i\, \forall i\neq j$, and $a_j\neq b_j$. Then, it is possible to choose weights $w_{k,l}$, $k\in\{1,\ldots,n\}$, and $c_l$ such that the following equations are satisfied with arbitrary accuracy: $P(v_l=h_l|h)= 1 \forall h\not\in \{a, b\}$, while $P(v_l=1|h=a)=p_a$ and $P(v_l=1|h=b)=p_b$ with arbitrary $p_a,p_b$.
\end{thmc}

By this Theorem, a sharing step can be accomplished in only one layer, whereas probability mass is transferred from a chosen vector to another vector differing in one entry. Futhermore, it demands adaptation only of the connection weights and offset weight of one single unit. Thereby, the overlay of a number of sharing steps in each layer is possible. 

The main idea in \cite{Roux:10} was to exploit these circumstances using a clever sequence of transactions of probabilities. 
The requirements for the realizability of sharing sequences using Theorem~2 in~\cite{Roux:10} can be summarized in properties of sequences of vectors. These properties are described in Theorem~3 of \cite{Roux:10}, or in the items 2-3 of our appropriately modified version of that Theorem, Lemma~\ref{lemmaGray} below. 

How the Theorem~2 in~\cite{Roux:10} and Lemma~\ref{lemmaGray} brace the construction of a universal DBN approximator will become clearer in the afterwards following Lemma~\ref{lemmaDBN}. 

\begin{lemma}\label{lemmaGray}
Let $n=\frac{2^b}{2} + b, b\in{\bf N}, b\geq 1$. There exist $a:=2^b=2(n-b)$ sequences of binary vectors $S_{i}$, $0\leq i< a-1$ composed of vectors $S_{i,k}, 1\leq k\leq\frac{2^n}{a}$ satisfying the following:
\begin{enumerate}
\item $\{S_0,\ldots,S_{a-1}\}$ is a partition of $\{0,1\}^n$.
\item $\forall i\in\{0,\ldots,a-1\}$, $\forall k\in\{1,\ldots,\frac{2^n}{a}-1\}$ we have $H(S_{i,k},S_{i,k+1})=1$, where $H(\cdot,\cdot)$ denotes the Hamming distance. 
\item $\forall i,j\in\{0,\ldots,a-1\}$ such that $i\neq j$ and $\forall k\in\{1,\ldots,\frac{2^n}{a}-1\}$ the bit switched between $S_{i,k}$ and $S_{j,k+1}$ and the bit switched between $S_{j,k}$ and $S_{j,k+1}$ are different, unless $H(S_{i,k},S_{j,k})=1$.
\end{enumerate}
\end{lemma}

\begin{proof}[Proof of Lemma~\ref{lemmaGray}]
Let $G_{n-b}^0$ be any Gray code for $(n-b)$ bits. 
Such a Gray code is a matrix of size $2^{n-b}\times (n-b)$, where every two consecutive rows have Hamming distance one to each other, and the collection of all rows is $\{0,1\}^{n-b}$. Obviously any permutation of columns of this Gray code has the same properties. Let $G_{n-b}^i$ be the cyclic permutation of columns $i$ positions to the left. 

Now define $S_i:=\begin{pmatrix} \text{bin}(i)\\ \vdots & G_{n-b}^{i \text{mod}(n-b)}\\ \text{bin}(i) \end{pmatrix}$, i.e. the first $b$ bits of the vector $S_{i,k}$ contain the $b$-bit binary representation of $i$. The rest of the bits contain the $k$-th row in the Gray code $G_{n-b}^0$ for arrays of length $n-b$ cyclically shifted $i$ positions to the left. 
The cyclic permutation makes that every two sequences of vectors $S_i$ and $S_j$, $i\neq j$ change the same bit in the same row (in this case they also do in every row) only if the value of the first part $\text{bin($i$)}$ and $\text{bin($j$)}$ of the two sequences differs in only one entry (in the first entry). 
\end{proof}

Every two consecutive vectors in a sequence given in Lemma~\ref{lemmaGray} differ in only one entry and this entry can be located in almost any position $\{1,\ldots,n\}$. In contrast, for the sequences given in Theorem~3 of~\cite{Roux:10} that entry can be located only in a subset of $\{1,\ldots,n\}$ of cardinality $n/2$. 

In the Lemma above, for any row, every one of $n-b$ entries is flipped by exactly two sequences. 
Regard that the attempt to produce $2n$ instead of $2(n-b)$ sequences with the properties 1-2 of the Lemma (and flips in all entries) would correspond to the following: 
Set $\begin{pmatrix}S_1\\ \vdots \\ S_{2n}\end{pmatrix} = G_n$, i.e., the sequences to be overlayed are portions of the same Gray code. In this case it is difficult to achive that condition 3. is satistfied, i.e., that if $S_i$ and $S_j$ flip the same bit in the same row, then $H(S_{i,k},S_{j,k})=1$. The condition 3. however is essential for the use of Theorem~2 of~\cite{Roux:10}. 
Most common Gray codes flip some entries more often than other entries and can be discarded. Oher sequences referred to as {\em totally balanced Gray codes} flip all entries equally often and exist whenever $n$ is a power of $2$, but still a strong cyclicity condition would be required. % in order to have property 3. of Lemma~\ref{lemmaGray}. 
On account of this we say that the sequences given in Lemma~\ref{lemmaGray} allow optimal use of Theorem~2 of~\cite{Roux:10}. 

The following Lemma~\ref{lemmaDBN} is a transcription of Lemma~1 in~\cite{Roux:10} with replacements of indices according to our construction. 
The proof is an obvious transcription which we omit here. Denote by $h^i$ a state vector of the units in the hidden layer $i$, and denote by $h^0$ a visible state. 
\begin{lemma}\label{lemmaDBN}
Let $p^*$ be an arbitrary distribution on $\{0,1\}^n$. Consider a DBN with $\frac{2^n}{a}+1$ layers and the following properties: 
\begin{enumerate}
\item $\forall i\in\{0,\ldots,a-1\}$ the top RBM between $h^{\frac{2^n}{a}}$ and $h^{\frac{2^n}{a}-1}$ assigns probability $\sum_k p^*(S_{i,k})$ to $S_{i,1}$,
\item $\forall i\in\{0,\ldots,a-1\}$, $\forall k\in\{1,\ldots,\frac{2^n}{a}-1\}$
\begin{eqnarray*}
P(h^{\frac{2^n}{a}-(k+1)} = S_{i,k+1}|h^{\frac{2^n}{a}-k}=S_{i,k}) &=& \frac{\sum_{t=k+1}^{\frac{2^n}{a}} p^*(S_{i,t})}{\sum_{t=k}^{\frac{2^n}{a}} p^*(S_{i,t})}, \\
P(h^{\frac{2^n}{a}-(k+1)}=S_{i,k}|h^{\frac{2^n}{a}-k}=S_{i,k}) &=& \frac{p^*(S_{i,k})}{\sum_{t=k}^{\frac{2^n}{a}} p^*(S_{i,t})},
\end{eqnarray*}
\item $\forall k\in\{1,\ldots,\frac{2^n}{a}-1\}$ the DBN provides
\begin{equation*}
P(h^{\frac{2^n}{a}-(k+1)}=u| h^{\frac{2^n}{a}-k}=u)=1,\quad\forall u\not\in \cup_i \{S_{i,k}\}.
\end{equation*}
\end{enumerate}
Such a DBN has $p^*$ as its marginal visible distribution.
\end{lemma}

We conclude this section with the proof of Theorem~\ref{universalDBN} and some remarks:

\begin{proof} [Proof of Theorem~\ref{universalDBN}]
The proof is analogous to the Proof of Theorem 4 in \cite{Roux:10}. 
We just need to show the existence of a DBN with the properties of the DBN described in Lemma~\ref{lemmaDBN}. 
In view of Theorem~\ref{universalRBM} it is possible to achive that the top RBM assigns arbitrary probability to the collection of vectors $S_{i,1}, i\in\{0,\ldots,a-1\}$, whenever it can be arranged in pairs of neighbouring vectors (or from Corollary~\ref{corollaryRBM}, if all vectors are equal in a set of entries). This requirement is met for $S_{i,1}, i\in\{0,\ldots,a-1\}$ of Lemma~\ref{lemmaGray}, (e.g. choosing a Gray code whose first element is $(0,\ldots,0)$ or $(1,\ldots,1)$). The subsequent layers are just like in the Proof of Theorem 4  in \cite{Roux:10}. They are possible in consideration of the mantained validity of Theorem~2 in~\cite{Roux:10} using the sequences provided in Lemma~\ref{lemmaGray} of the present paper. 
The only difference is that by our definition of $S_{i}$, $i\in\{0,\ldots,a-1\}$, at each layer $n-b$ bit flips (with correct probabilities) occur, instead of $\tfrac{n}{2}$. 
\end{proof}

In the paper \cite{Roux:10} the authors overlayed $n$ sequences of sharing steps (Theorem~3 in that paper) for the construction of a universal DBN approximator. 
In principle an overlay of more such sequences is possible. %$2n$ sharing steps in each layer is possible, (even if the overlay of $2n$ optimal sequences of sharing steps is not, as discussed below Lemma~\ref{lemmaGray} in the present paper). 
This is what we exploit in our proof, (the sequences given in Lemma~\ref{lemmaGray}). 
Apparently, the overlay of more sequences was not realized in that paper because for the initialization of these sequences, (property 1. in Lemma~1 in that paper), the authors use Theorem~2 of \cite{Roux:08}, which only allows to assign arbitrary probability to $n$ vectors. Our result Theorem~\ref{universalRBM} overcomes this difficulty and allows to initialize up to $2(n+1)$ sequences, which we use to obtain property 1. in Lemma~\ref{lemmaDBN}.

\section{Conclusion}

We have shown that a Deep Belief Network (DBN) with $\frac{2^n}{2(n-b)}$, $b\sim\log n$, hidden layers of size $n$ is capable of approximating any distribution on $\{0,1\}^n$ arbitrarily well as its marginal visible distribution. (This confirms a conjecture presented in \cite{Roux:10}). The number of layers $\frac{2^n}{2(n-b)}$ is of order $\frac{2^n}{2n}$. 
This DBN has $\frac{2^n}{2(n-b)} n^2 + \frac{2^n}{2(n-b)} n + n$ parameters, which is of order $\frac{n 2^n}{2}$. 

Furthermore, we have shown that a Restricted Boltzmann Machine (RBM) with $\frac{2^n}{2}-1$ hidden units is capable of approximating any distribution on $\{0,1\}^n$ arbitrarily well as its marginal visible distribution. 
This RBM has $\frac{2^n}{2} n + \frac{2^n}{2} $ parameters, which is of order $\frac{n2^n}{2}$. 

Our results improve all to date known bounds on the minimal size of universal DBN and RBN approximators. 
We still do not know if our results represent the minimal sufficient size for universal DBN and RBN approximators. 
Our construction already exploits Theorem~2 in~\cite{Roux:10} exhaustively, and therefore a construction using only similar ideas will not allow improvements. 
However, we have performed numerical computations (we do not include details here) showing the existence of RBM's containing less than $\frac{2^n}{2}-1$ hidden units and which can approximate complex classes of distributions on $\{0,1\}^n$ arbitrarily well. This suggests that in the present construction the representational power of RBM's is not fully exploited. Whether further reductions of the size of a universal DBN approximator are possible is subject of our ongoing research, \cite{guido:2010a}.

\end{document}